\documentclass[12pt]{article}
\usepackage[colorlinks]{hyperref}   
\usepackage{float} 
\usepackage{mwe}
\usepackage{color}
\usepackage{graphicx,subfigure,amsmath,amssymb,amsfonts,bm,epsfig,epsf,url,dsfont}
\usepackage{fullpage}
\usepackage[small,bf]{caption}
\usepackage[top=1in,bottom=1in,left=1in,right=1in]{geometry}
\usepackage[sort&compress]{natbib} 
\usepackage{bbm}      
\usepackage{booktabs}
\usepackage{cases}
\usepackage{dsfont}
\usepackage{multirow}
\usepackage{graphicx}

\newtheorem{theorem}{Theorem}
\newtheorem{conjecture}{Conjecture}

\newtheorem{lemma}{Lemma}

\newenvironment{proof}{\medskip\noindent{\bf Proof.}}{\hfill$\Box$\vspace*{1mm}\medskip}

\renewcommand{\phi}{\varphi}

\renewcommand{\P}{\mathbb{P}}
\newcommand{\E}{\mathbb{E}}

\newcommand{\R}{\mathbb{R}}

\newcommand{\cN}{\mathcal{N}}

\newcommand{\cF}{\mathcal{F}}

\newcommand{\cS}{\mathcal{S}}

\def\ds1{\mathds{1}}
\renewcommand{\epsilon}{\varepsilon}

\newlength{\minipagewidth}
\newcommand{\beq}{\begin{equation}}
\newcommand{\eeq}{\end{equation}}

\newcommand{\beqa}{\begin{eqnarray}}
\newcommand{\eeqa}{\end{eqnarray}}

\newcommand{\beqan}{\begin{eqnarray*}}
\newcommand{\eeqan}{\end{eqnarray*}}

\def\ba#1\ea{\begin{align*}#1\end{align*}} 
\def\banum#1\eanum{\begin{align}#1\end{align}} 

\newcommand{\mS}{\mathbb{S}}
\newcommand{\Lip}{\mathrm{Lip}}

\begin{document}
\title{A law of robustness for two-layers neural networks}
\author{S\'ebastien Bubeck \\
Microsoft Research
\and Yuanzhi Li \thanks{This work was partly done while Y. Li and D. Nagaraj were visiting Microsoft Research.} \\
CMU
\and Dheeraj Nagaraj\footnotemark[1]\\
MIT}

\maketitle

\begin{abstract}
We initiate the study of the inherent tradeoffs between the size of a neural network and its robustness, as measured by its Lipschitz constant. We make a precise conjecture that, for any Lipschitz activation function and for most datasets, any two-layers neural network with $k$ neurons that perfectly fit the data must have its Lipschitz constant larger (up to a constant) than $\sqrt{n/k}$ where $n$ is the number of datapoints. In particular, this conjecture implies that overparametrization is necessary for robustness, since it means that one needs roughly one neuron per datapoint to ensure a $O(1)$-Lipschitz network, while mere data fitting of $d$-dimensional data requires only one neuron per $d$ datapoints. We prove a weaker version of this conjecture when the Lipschitz constant is replaced by an upper bound on it based on the spectral norm of the weight matrix. We also prove the conjecture in the high-dimensional regime $n \approx d$ (which we also refer to as the undercomplete case, since only $k \leq d$ is relevant here). Finally we prove the conjecture for polynomial activation functions of degree $p$ when $n \approx d^p$. We complement these findings with experimental evidence supporting the conjecture.
\end{abstract}

\section{Introduction} \label{sec:intro}
We study two-layers neural networks with inputs in $\R^d$, $k$ neurons, and Lipschitz non-linearity $\psi : \R \rightarrow \R$. These are functions of the form:
\begin{equation} \label{eq:nnform}
x \mapsto \sum_{\ell=1}^k a_\ell \psi(w_\ell \cdot x + b_\ell) \,,
\end{equation}
with $a_\ell, b_\ell \in \R$ and $w_\ell \in \R^d$ for any $\ell \in [k]$. We denote by $\cF_k(\psi)$ the set of functions of the form \eqref{eq:nnform}. When $k$ is large enough and $\psi$ is non-polynomial, this set of functions can be used to fit any given data set \citep{cybenko1989approximation, leshno1993multilayer}. That is, given a data set $(x_i, y_i)_{i \in [n]} \in (\R^d \times \R)^n$, one can find $f \in \cF_k(\psi)$ such that
\begin{equation} \label{eq}
f(x_i) = y_i, \forall i \in [n] \,.
\end{equation}
In a variety of scenarios one is furthermore interested in fitting the data {\em smoothly}. For example, in machine learning, the data fitting model $f$ is used to make predictions at unseen points $x \not\in \{x_1, \hdots, x_n\}$. It is reasonable to ask for these predictions to be stable, that is a small perturbation of $x$ should result in a small perturbation of $f(x)$. 

A natural question is: how ``costly'' is this stability restriction compared to mere data fitting? In practice it seems much harder to find robust models for large scale problems, as first evidenced in the seminal paper \citep{G14}. In theory the ``cost'' of finding robust models has been investigated from a computational complexity perspective in \citep{bubeck2019adversarial}, from a statistical perspective in \citep{Schmidt18}, and more generally from a model complexity perspective in \citep{pmlr-v99-degwekar19a, raghunathanadversarial, allen2020feature}. We propose here a different angle of study within the broad model complexity perspective: does a model {\em have to} be larger for it to be robust? Empirical evidence (e.g., \citep{Goodfellow15, Madry18}) suggests that bigger models (also known as ``overparametrization'') do indeed help for robustness.

Our main contribution is a conjecture (Conjecture \ref{conj:LB} and Conjecture \ref{conj:UB}) on the precise tradeoffs between size of the model (i.e., the number of neurons $k$) and robustness (i.e., the Lipschitz constant of the data fitting model $f \in \cF_k(\psi))$ for generic data sets. 
We say that a data set $(x_i, y_i)_{i \in [n]}$ is {\em generic} if it is i.i.d. with $x_i$ uniform (or approximately so, see below) on the sphere $\mS^{d-1} = \{x \in \R^d : \|x\|=1\}$ and $y_i$ uniform on $\{-1,+1\}$. We give the precise conjecture in Section \ref{sec:conj}. 
We prove several weaker versions of Conjecture \ref{conj:LB} and Conjecture \ref{conj:UB} respectively in Section \ref{sec:LB} and Section \ref{sec:interpolation}. We also give empirical evidence for the conjecture in Section \ref{sec:exp}.

\paragraph{A corollary of our conjecture.} A key fact about generic data, established in \cite{baum1988capabilities, yun2019small, BELM20}, is that one can memorize arbitrary labels with $k \approx n/d$, that is merely one neuron per $d$ datapoints. Our conjecture implies that for such optimal-size neural networks it is {\em impossible} to be robust, in the sense that the Lipschitz constant must be of order $\sqrt{d}$. The conjecture also states that to be robust (i.e. attain Lipschitz constant $O(1)$) one must {\em necessarily} have $k \approx n$, that is roughly each datapoint must have its own neuron. Therefore, we obtain a trade off between size and robustness, namely to make the network robust it needs to be {\em $d$ times larger} than for mere data fitting. We illustrate these two cases in Figure \ref{fig:intro}. We train a neural network to fit generic data, and plot the maximum gradient over several randomly drawn points (a proxy for the Lipschitz constant) for various values of $\sqrt{d}$, when either $k=n$ (blue dots) or $k= \frac{10 n}{d}$ (red dots). As predicted, for the large neural network ($k=n$) the Lipschitz constant remains roughly constant, while for the optimally-sized one ($k= \frac{10 n}{d}$) the Lipschitz constant increases roughly linearly in $\sqrt{d}$.

\begin{figure}
        \centering
        \includegraphics[width=0.5\textwidth]{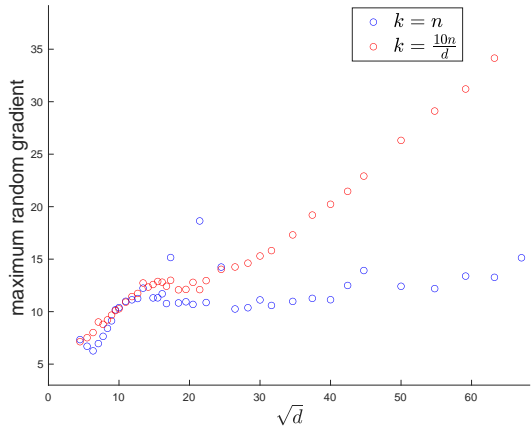} 
        \caption{See Section \ref{sec:exp} for the details of this experiment.}
        \label{fig:intro}
\end{figure}

\paragraph{Notation.} For $\Omega \subset \R^d$ we define $\Lip_{\Omega}(f) = \sup_{x\neq x' \in \Omega} \frac{|f(x)-f(x')|}{\|x-x'\|}$ (if $\Omega = \R^d$ we omit the subscript and write $\Lip(f)$), where $\|\cdot\|$ denotes the Euclidean norm. For matrices we use $\|\cdot\|_{\mathrm{op}}, \|\cdot\|_{\mathrm{op},*}$, $\|\cdot\|_{\mathrm{F}}$ and $\langle \cdot, \cdot \rangle$ for respectively the operator norm, the nuclear norm (sum of singular values), the Frobenius norm, and the Frobenius inner product. We also use these notations for tensors of higher order, see Appendix \ref{app1} for more details on tensors. We denote $c>0$ and $C>0$ for universal numerical constants, respectively small enough and large enough, whose values can change in different occurences. Similarly, by $c_p > 0$ and $C_p > 0$ we denote constants depending only on the parameter $p$. We also write $\mathrm{ReLU}(t) = \max(t,0)$ for the rectified linear unit.

\paragraph{Generic data.} We give some flexibility in our definition of ``generic data'' in order to focus on the essence of the problem, rather than technical details. Namely, in addition to the spherical model mentioned above, where $x_i$ is i.i.d. uniform on the sphere $\mS^{d-1} = \{x \in \R^d : \|x\|=1\}$, we also consider the very closely related model where $x_i$ is i.i.d. from a centered Gaussian with covariance $\frac{1}{d} \mathrm{I}_d$ (in particular $\E[ \|x_i\|^2 ] =1$, and in fact $\|x_i\|$ is tightly concentrated around $1$). In both cases we consider $y_i$ to be i.i.d. random signs. We say that a property holds with high probability for {\em generic data}, if it holds with high probability either for the spherical model or for the Gaussian model.

\section{A conjectured law of robustness} \label{sec:conj}
Our main contribution is the following conjecture, which asserts that, on generic data sets, increasing the size of a network is necessary to obtain robustness:

\begin{conjecture} \label{conj:LB}
For generic data sets, with high probability\footnote{We do not quantify the ``with high probability'' in our conjecture. We believe the conjecture to be true except for an event of exponentially small probability with respect to the sampling of a generic data set, but even proving that the statement is true with strictly positive probability would be extremely interesting.}, any $f \in \cF_k(\psi)$ fitting the data\footnote{We expect the same lower bound to hold even if one only asks $f$ to approximately fit the data. In fact our provable variants of Conjecture \ref{conj:LB} are based proofs that are robust to only assuming an approximately fitting $f$.} (i.e., satisfying \eqref{eq}) must also satisfy:
\[
\Lip_{\mS^{d-1}}(f) \geq c \sqrt{\frac{n}{k}} \,.
\]
\end{conjecture}
Note that for generic data, with high probability (for $n= \mathrm{poly}(d)$), there exists a smooth interpolation. Namely there exists $g : \R^d \rightarrow \R$ with $g(x_i) = y_i, \forall i \in [n]$ and $\Lip(g) = O(1)$. This follows easily from the fact that with high probability (for large $d$) one has $\|x_i - x_j\| \geq 1, \forall i \neq j$. Conjecture \ref{conj:LB} puts restrictions on how smoothly one can interpolate data with small neural networks.  A striking consequence of the conjecture is that for a two-layers neural network $f \in \cF_k(\psi)$ to be as robust as this function $g$ (i.e., $\Lip(f) = O(1)$) and fit the data, one must have $k = \Omega(n)$, i.e., roughly one neuron per data point. On the other hand with that many neurons it is quite trivial to smoothly interpolate the data, as we explain in Section \ref{sec:ReLUinterpolation}. Thus the conjecture makes a strong statement that essentially the trivial smooth interpolation is the best thing one can do. In addition to making the prediction that one neuron per datapoint is necessary for optimal smoothness, the conjecture also gives a precise prediction on the possible tradeoff between size of the network and its robustness. We also conjecture that this whole range of tradeoffs is actually achievable:

\begin{conjecture} \label{conj:UB}
Let $n, d, k$ be such that $C \cdot \frac{n}{d} \leq k \leq C \cdot n$ and $n \leq d^C$ where $C$ is an arbitrarily large constant in the latter occurence. There exists $\psi$ such that, for generic data sets, with high probability, there exists $f \in \cF_k(\psi)$ fitting the data (i.e., satisfying \eqref{eq}) and such that
\[
\Lip_{\mS^{d-1}}(f) \leq C \sqrt{\frac{n}{k}} \,.
\]
\end{conjecture}
The condition $k \leq C \cdot n$ in Conjecture \ref{conj:UB} is necessary, for any interpolation of the data must have Lipschitz constant at least a constant. The other condition on $k$, namely $k \geq C \cdot \frac{n}{d}$, is also necessary, for that many neurons is needed to merely guarantee the existence of a data-fitting neural network with $k$ neurons (see \cite{baum1988capabilities, yun2019small, BELM20}). Finally the condition $n \leq d^C$ is merely used to avoid explicitly stating a logarithmic term in our conjecture (indeed, equivalently one can replace this condition by adding a multiplicative polylogarithmic term in $d$ in the claimed inequality).

\paragraph{Two extreme regimes.}
Two regimes of particular interest are the {\em optimal smoothness regime}, where we consider how many neurons we need to achieve $\Lip(f) = O(1)$, and the {\em optimal size regime}, where we consider how small a Lipschitz constant is achievable with $k = C \cdot n / d$ (i.e., the smallest number of neurons needed to merely fit the data). Our conjectures predict that to be in the optimal smoothness regime it is necessary and sufficient to have $k \approx n$, while for optimal size regime it is necessary and sufficient to have $\Lip(f) \approx \sqrt{d}$.

\paragraph{Our results around Conjecture \ref{conj:UB} (Section \ref{sec:interpolation}).} We prove Conjecture \ref{conj:UB} for both the optimal smoothness regime (which is quite straightforward, see Section \ref{sec:ReLUinterpolation}) and for the optimal size regime (here more work is needed, and we use a certain tensor-based construction, see Section \ref{sec:tensorinterpolation}). In the latter case we only prove approximate data fitting (mostly to simplify the proofs), and more importantly we need to assume that $n$ is of order $d^p$ for some even integer $p$. It would be interesting to generalize the proof to any $n$. While the conjecture remains open between these two extreme regimes, we do give a construction in Section \ref{sec:ReLUinterpolation} which has the correct qualitative behavior (namely increasing $k$ improves the Lipschitz constant), albeit the scaling we obtain is $n/k$ instead of $\sqrt{n/k}$, see Theorem \ref{thm:weakconjUB}.

\paragraph{Our results around Conjecture \ref{conj:LB} (Section \ref{sec:LB}).} We prove a weaker version of Conjecture \ref{conj:LB} where the Lipschitz constant on the sphere is replaced by a proxy involving the spectral norm of the weight matrix, see Theorem \ref{thm:spectralnorm}. We also prove the conjecture in the optimal size regime, specifically when $n=d^p$ for an integer $p$ and one uses a polynomial activation function of degree $p$, see Theorem \ref{thm:polynomialLB}. For $p=1$ (i.e., $n \approx d$) we in fact prove the conjecture for abritrary non-linearities, see Theorem \ref{thm:undercomplete}.

\paragraph{Further open problems.} Our proposed law of robustness is a first mathematical formalization of the broader phenomenon that ``overparametrization in neural networks is necessary for robustness''. Ideally one would like a much more refined understanding of the phenomenon than the one given in Conjecture \ref{conj:LB}. For example, one could imagine that in greater generality, the law would read $\Lip_{\Omega}(f) \geq F(k, (x_i,y_i)_{i \in [n]}, \Omega)$. That is, we would like to understand how the achievable level of smoothness depends on the particular data set at hand, but also on the set where we expect to be making predictions. Another direction to generalize the law would be to extend it to multi-layers neural networks. In particular one could imagine the most general law would replace the parameter $k$ (number of neurons) by the type of architecture being used and in turn predict the best architecture for a given data set and prediction set. Finally note that our proposed law apply to {\em all} neural networks, but it would also be interesting to understand how the law interacts with algorithmic considerations (for example in Section \ref{sec:exp} we use Adam~\cite{kingma2014adam} to find a set of weights that qualitatively match Conjecture \ref{conj:UB}).

\section{Smooth interpolation} \label{sec:interpolation}
We start with a warm-up in Section \ref{sec:simplest} where we discuss the simplest case of interpolation with a linear model ($k=1, n\leq d$) and in Section \ref{sec:simplest2} for the optimal smoothness regime ($k=n$). We generalize the construction of Section \ref{sec:simplest2} in Section \ref{sec:ReLUinterpolation} to obtain the whole range of tradeoffs between $k$ and $\Lip(f)$, albeit with a suboptimal scaling, see Theorem \ref{thm:weakconjUB}. We also generalize the linear model calculations of Section \ref{sec:simplest} in Section \ref{sec:tensorinterpolation} to obtain the optimal size regime for larger values of $n$ via a certain tensor construction.

\subsection{The simplest case: optimal size regime when $n \leq c \cdot d$} \label{sec:simplest}
Let us consider $k=1$, $n \leq c \cdot d$ and $\psi(t) = t$. Thus we are trying to find $w \in \R^d$ such that $w \cdot x_i = y_i$ for all $i \in [n]$, or in other words $X w = Y$ with $X$ the $n \times d$ matrix whose $i^{th}$ row is $x_i$, and $Y=(y_1,\hdots, y_n)$. The smoothest solution to this system (i.e., the one minimizing $\|w\|$) is 
\[
w = X^{\top} (X X^{\top})^{-1} Y \,,
\]
Note that 
\[
\Lip(x \mapsto w \cdot x) = \|w\| = \sqrt{w^{\top} w} = \sqrt{Y^{\top} (X X^{\top})^{-1} Y} \,.
\]
Using [Theorem 5.58, \cite{vershynin12}] one has with probability at least $1-\exp(C - c d)$ (and using that $n \leq c \cdot d$) that
\[
X X^{\top} \succeq \frac{1}{2} I_n \,,
\]
and thus $\Lip(x \mapsto w \cdot x) \leq \sqrt{2} \cdot \|Y\| = \sqrt{2 n}$. This concludes the proof sketch of Conjecture \ref{conj:UB} for the simplest case $k=1$ and $n \leq d$.

\subsection{Another simple case: optimal smoothness regime} \label{sec:simplest2}
\sloppy
Next we consider the optimal smoothness regime in Conjecture \ref{conj:UB}, namely $k=n$. First note that, for generic data and $n= \mathrm{poly}(d)$, with high probability the caps $C_i := \left\{x \in \mS^{d-1} : x_i \cdot x \geq 0.9 \right\}$ are disjoint sets and moreover they each contain a single data point (namely $x_i$). With a single ReLU unit it is then easy to make a smooth function ($10$-Lipschitz) which is $0$ outside of $C_i$ and equal to $+1$ at $x_i$ (in other words the neuron activates for a single data point), namely $x \mapsto 10 \cdot \mathrm{ReLU}\left(x_i \cdot x - 0.9 \right)$. Thus one can fit the entire data set with the following ReLU network which is $10$-Lipschitz on the sphere:
\[
f(x) = \sum_{i=1}^n 10 y_i \cdot \mathrm{ReLU}\left(x_i \cdot x - 0.9 \right) \,.
\]
This concludes the proof of Conjecture \ref{conj:UB} for the optimal smoothness regime $k=n$.

\subsection{Intermediate regimes via ReLU networks} \label{sec:ReLUinterpolation}
We now combine the two constructions above (the linear model of Section \ref{sec:simplest} and the ``isolation" strategy of Section \ref{sec:simplest2}) to give a construction that can trade off size for robustness (albeit not optimally according to Conjecture \ref{conj:UB}):

\begin{theorem} \label{thm:weakconjUB}
Let $n, d, k$ be such that $C \cdot \frac{n \log(n)}{d} \leq k \leq C \cdot n$. For generic data sets, with probability at least $1- 1/n^C$, there exists $f \in \cF_k(\mathrm{ReLU})$ fitting the data (i.e., satisfying \eqref{eq}) and such that
\[
\Lip_{\mS^{d-1}}(f) \leq C \cdot \frac{n \log(d)}{k} \,.
\]
\end{theorem}

\begin{proof}
Let $m = \frac{n}{k}$ (by assumption $m \leq c \cdot \frac{d}{\log(n)}$) and assume it is an integer. Let us choose $m$ points with the same label, say it is the points $x_1, \hdots, x_m$ with label $+1$. As in Section \ref{sec:simplest} let $w \in \R^d$ be the minimal norm vector that satisfy $w \cdot x_i = 1$, and thus as we proved there with probability at least $1- \exp(C - c d)$ one has $\|w\| \leq \sqrt{2 m}$. Crucially for the end of the proof, also note that the distribution of $w$ is rotationally invariant.
Next observe that with probability at least $1- 1/n^C$ (with respect to the sampling of $x_{m+1}, \hdots, x_n$) one has $\max_{i \in \{m+1, \hdots, n\}} |w \cdot x_i| \leq C \cdot \|w\| \sqrt{\frac{\log(n)}{d}} \leq \frac{1}{2}$. In particular the cap $\mathcal{C} := \{x \in \mS^{d-1} : w \cdot x \geq \frac{1}{2} \}$ contains $x_1, \hdots, x_m$ but does not contain any $x_i$, $i>m$.
Thus the neuron
\[
x \mapsto 2 \cdot \mathrm{ReLU}\left(w \cdot x - \frac{1}{2}\right) \,,
\]
computes the value $1$ at points $x_1, \hdots, x_m$ and the value $0$ at the rest of the training set.
\newline

One can now repeat this process, and build the neurons $w_1, \hdots, w_{k}$ (all with norm $\leq \sqrt{2 m}$), so that (with well-chosen signs $\xi_{\ell} \in \{-1,1\}$) the data is perfectly fitted by the function:
\[
f(x) = \sum_{\ell=1}^k 2 \cdot \xi_{\ell} \cdot \mathrm{ReLU}\left(w_{\ell} \cdot x - \frac{1}{2}\right) \,.
\]
It only remains to estimate the Lipschitz constant. Note that if a point $x \in \mathbb{S}^{d-1}$ activates a certain subset $A \subset \{1,\hdots, k\}$ of the neurons, then the gradient at this point is $\sum_{\ell \in A} w_{\ell}'$ with $w_{\ell}' = 2 \xi_{\ell} w_{\ell}$. Using that the $w_{\ell'}$ are rotationally invariant, one also has with probability at least $1- C n^2 \exp(- c d)$ that $\left\| \sum_{\ell \in A} w_{\ell}' \right\|^2 \leq C \cdot |A| \cdot m$ for all $A \subset \{1,\hdots, k\}$. Thus it only remains to control how large $A$ can be. We show below that $|A| \leq C m \log(d)$ with probability at least $1- C \exp(- c d \log(d))$ which will conclude the proof.
\newline

If $x$ activates neuron $\ell$ then $w_{\ell} \cdot x \geq \frac12 \geq \frac{\|w_{\ell}\|}{4 \sqrt{m}}$. Now note that for any fixed $x \in \mS^{d-1}$ and fixed $A \subset [k]$, $\P\left( \forall \ell \in A, w_{\ell} \cdot x \geq \frac{\|w_{\ell}\|}{4 \sqrt{m}} \right) \leq C \exp\left( - c |A| \frac{d}{m} \right)$, so that
\[
\P\left(\exists A \subset [k] : |A| = a \text{ and } \forall \ell \in A, w_{\ell} \cdot x \geq \frac{\|w_{\ell}\|}{4 \sqrt{m}} \right) \leq \exp \left(C a \log(k) -  c a \frac{d}{m} \right) \,.
\]
In particular we conclude that with $a = C m \log(d)$ the probability that a fixed point on the sphere activates more than $a$ neuron is exponentially small in $d \log(d)$ (recall that $m \log(k) \leq c d$ by assumption). Thus we can conclude via an union bound on an $\epsilon$-net that the same holds for the entire sphere simultaneously. This concludes the proof.
\end{proof}

\subsection{Optimal size networks via tensor interpolation} \label{sec:tensorinterpolation}
In this section we essentially prove Conjecture \ref{conj:UB} in the optimal size regime (namely $k \cdot d \approx n$), with three caveats:
\begin{enumerate}
\item We allow a slack of a $\log n$ factor by considering $k \cdot d = C n \log (n)$ instead of the optimal $k \cdot d = C n$ as in \cite{baum1988capabilities, BELM20}.
\item We only prove approximate fit rather than exact fit. It is likely that with more work one can use the core of our argument to obtain exact fit. For that reason we did not make any attempt to optimize the dependency on $\epsilon$ in Theorem \ref{thm:tensorUB}. For instance one could probably obtain $\log(1/\epsilon)$ rather than $1/\mathrm{poly}(\epsilon)$ dependency by using an iterative scheme that fits the residuals, as in \citep{bresler2020corrective,BELM20}.
\item Finally we have to assume that $n$ is of order $d^p$ for some even integer $p$. While it might be that one can apply the same proof for odd integers, the whole construction crucially relies on $p$ being an even integer as we essentially do a linear regression over the feature embedding $x \mapsto x^{\otimes p}$. A possible approach to extend the proof to other values of $n$ would be to use the scheme of Section \ref{sec:ReLUinterpolation} with the linear regression there replaced by the tensor regression of the present section.
\end{enumerate}

\begin{theorem} \label{thm:tensorUB}
Fix $\epsilon >0$, $p$ an even integer, and let $\psi(t) = t^p$. Let $n,d,k$ be such that $n \log(n) = \epsilon^2 \cdot d^p$ and $k = C_p \cdot d^{p-1}$. Then for generic data, with probability at least $1- 1/n^C$, there exists $f \in \cF_k(\psi)$ such that
\begin{equation} \label{eq:approximatefit}
|f(x_i) - y_i| \leq C_p \cdot \epsilon \,, \forall i \in [n] \,,
\end{equation}
and
\[ 
\mathrm{Lip}_{\mS^{d-1}}(f) \leq C_p \sqrt{\frac{n}{k}} \,.
\]
\end{theorem}

\begin{proof}
We propose to approximately fit with the following neural network:
\[
f(x) = \sum_{i=1}^n y_i (x_i \cdot x)^p \,.
\]
Naively one might think that this neural network requires $n$ neurons. However, it turns out that one can always decompose a symmetric tensor of order $p$ into $k=2^p d^{p-1}$ rank-$1$ symmetric tensors of order $p$, so that in fact $f \in \cF_k(\psi)$. For $p=2$ this simply follows from eigendecomposition and for general $p$ we give a simple proof in [Appendix \ref{app1}, Lemma \ref{lem:tensordecomposition}]. 

One also has by applying [Appendix \ref{app2}, Lemma \ref{lem:conctensor}] with $\tau = C_p \log(n)$ and doing an union bound, that with probability at least $1-1/n^C$, for any $j \in [n]$,
\[
\left| \sum_{i=1, i \neq j}^n y_i (x_i \cdot x_j)^p \right| \leq C_p \sqrt{\frac{n \log(n)}{d^{p}}} \leq C_p \epsilon \,.
\]
In particular this proves \eqref{eq:approximatefit}.

Thus it only remains to estimate the Lipschitz constant, which by [Appendix \ref{app1}, Lemma \ref{lem:liptensor}] is reduced to estimating the operator norm of the tensor $\sum_{i=1}^n y_i x_i^{\otimes p}$. We do so in [Appendix \ref{app2}, Lemma \ref{lem:operatornormtensor}].
\end{proof}

\section{Provable weaker versions of Conjecture \ref{conj:LB}} \label{sec:LB}
Conjecture \ref{conj:LB} can be made weaker along several directions. For example the quantity of interest $\Lip_{\mS^{d-1}}(f)$ can be replaced by various upper bound proxies for the Lipschitz constant. A mild weakening would be to replace it by the Lipschitz constant on the whole space (we shall in fact only consider this notion here). A much more severe weakening is to replace it by a quantity that depends on the spectral norm of the weight matrix (essentially ignoring the pattern of activation functions). For the latter proxy we actually give a complete proof, see Theorem \ref{thm:spectralnorm}, which in particular formally proves that ``overparametrization is a law of robustness for generic data sets''. Other interesting directions to weaken the conjecture include specializing it to common activation functions, or simply having a smaller lower bound on the Lipschitz constant. In Section \ref{sec:undercomplete} we prove the conjecture when $n$ is replaced by $d$ in the lower bound. We say that this inequality is in the ``very high-dimensional case'', in the sense that it matches the conjecture for $n \approx d$ (alternatively we also refer to it as the ``undercomplete case'', in the sense that only $k \leq d$ is relevant in this very high-dimensional scenario). In the moderately high-dimensional case ($n \gg d$) the proof strategy we propose in Section \ref{sec:undercomplete} cannot work. In Section \ref{sec:tensor} we give another argument for the latter case, specifically in the optimal size regime (i.e., $k \cdot d \approx n$) and for a power activation function, see Theorem \ref{thm:tensorLB}. We generalize this to polynomial activation functions in Section \ref{sec:polynomial}. In the specific case of a quadratic activation function we also show a lower bound that applies for any $k$ and which is in fact larger than the one given in Conjecture \ref{conj:LB}, see Theorem \ref{thm:quadratic} in Section \ref{sec:quadratic}.

\subsection{Spectral norm proxy for the Lipschitz constant} \label{sec:spectralnorm}
We can rewrite \eqref{eq:nnform} as
\begin{equation} \label{eq:nnform2}
f(x) = a^{\top} \psi(W x + b) \,,
\end{equation}
where $a=(a_1,\hdots, a_k) \in \R^k$, $b=(b_1,\hdots, b_k) \in \R^k$, $W \in \R^{k \times d}$ is the matrix whose $\ell^{th}$ row is $w_{\ell}$, and $\psi$ is extended from $\R \rightarrow \R$ to $\R^k \rightarrow \R^k$ by applying it coordinate-wise. We prove here the following:
\begin{theorem} \label{thm:spectralnorm}
Assume that $\psi$ is $L$-Lipschitz. For $f \in \cF_k(\psi)$ one has
\begin{equation} \label{eq:spectralnorm1}
\Lip(f) \leq L \cdot \|a\| \cdot \|W\|_{\mathrm{op}} \,.
\end{equation}
For a generic data set, if $f(x_i) = y_i, \forall i \in [n]$ and $f$ has no bias terms (i.e., $b=0$ in \eqref{eq:nnform2}), then with positive probability one has:
\begin{equation} \label{eq:spectralnorm2}
L \cdot \|a\| \cdot \|W\|_{\mathrm{op}} \geq \sqrt{\frac{n}{k}} \,.
\end{equation}
\end{theorem}
Note that we prove the inequality \eqref{eq:spectralnorm2} only with positive probability (i.e., there exists a data set where the inequality is true), but in fact it is easy to derive the statement with high probability using classical concentration inequalities.

\begin{proof}
Since $\psi : \R \rightarrow \R$ is $L$-Lipschitz, we have:
\[
f(x) - f(x') \leq \|a\| \cdot \|\psi(Wx +b) - \psi(W x' +b)\| \leq L \cdot \|a\| \cdot \|W x - W x'\| \leq L \cdot \|a\| \cdot \|W\|_{\mathrm{op}} \cdot \|x-x'\| \,,
\]
which directly proves \eqref{eq:spectralnorm1}.

Next, following the proof of [Proposition 1, \cite{BELM20}] one obtains that for a generic data set, with positive probability, one has (without bias terms):
\[
\sum_{\ell=1}^k |a_{\ell}| \cdot \|w_{\ell}\| \geq \frac{\sqrt{n}}{L} \,.
\]
It only remains to observe that:
\[
\frac{\sqrt{n}}{L} \leq \sum_{\ell=1}^k |a_{\ell}| \cdot \|w_{\ell}\| \leq \sqrt{\sum_{\ell=1}^k |a_{\ell}|^2 \cdot \sum_{\ell=1}^k \|w_{\ell}\|^2} = \|a\| \cdot \|W\|_{\mathrm{F}} \leq \sqrt{k} \cdot \|a\| \cdot \|W\|_{\mathrm{op}} \,,
\]
which concludes the proof of \eqref{eq:spectralnorm2}.
\end{proof}

\subsection{Undercomplete case} \label{sec:undercomplete}
Next we prove the conjecture in the high dimensional case $n \approx d$. More precisely we replace $n$ by $d$ in the conjectured lower bound. Importantly note that the resulting lower bound then becomes non-trivial only in the regime $k \leq d$ (the ``undercomplete case''). 

We consider in fact a slightly more general scenario than interpolation with a neural network, namely we simply assume that one interpolates the data with a function $f(x) = g(Px)$ where $P$ is a linear projection on a $k$-dimensional subspace (this clearly generalizes $f \in \cF_k(\psi)$, in fact it even allows for the non-linearity $\psi$ to depend on the data\footnote{It would be interesting to study whether allowing data-dependent non-linearities could affect the conclusion of our conjectures. Such study would need to crucially rely on having only one hidden layer, as it is known from the Kolmogorov-Arnold theorem that with two hidden layers and data-dependent non-linearities one can obtain perfect approximation properties with $k \leq d$ (albeit the non-linearities are non-smooth).}, or to have a different non-linearity for each neuron).

\begin{theorem} \label{thm:undercomplete}
Let $n \geq d$. Let $f: \R^d \rightarrow \R$ be a function such that $f(x_i) = y_i, \forall i \in [n]$ and moreover $f(x) = g(Px)$ for some differentiable function $g : \R^k \rightarrow \R$ and matrix $P \in \R^{k \times d}$. Then, for generic data, with probability at least $1-\exp(C - c d)$ one must have
\[
\mathrm{Lip}(f) \geq c \sqrt{\frac{d}{k}} \,.
\]
\end{theorem}

\begin{proof}
Let us modify $g$ so that $P$ is simply an orthogonal projection operator (i.e., $P P^{\top} = \mathrm{I}_k$). Let us also assume for sake of notational simplicity that we have a balanced data set of size $2n$, that is with: $y_1, \hdots y_{n} = +1$ and $y_{n+1}, \hdots, y_{2n} = -1$. Let us denote $x_i' = x_i - x_{n+i}$ for $i \in [n]$. The sequence $x_i'$ is i.i.d. and satisfies $\E[ x_i' x_i'^{\top} ] = \frac{2}{d} \mathrm{I}_d$.

Now observe that on the segment $[x_i,x_{n+i}]$ (whose length is less than $2$), the function $f$ changes value from $+1$ to $-1$, and thus there exists $z_i \in [x_i, x_{n+i}]$ such that:
\[
1 \leq | \nabla f(z_i) \cdot (x_i - x_{n+i}) | = |\nabla f(z_i) \cdot x_i'| \,. 
\]
Moreover one has (using that $\nabla f(x) = P^{\top} \nabla g(x)$, and thus $\|\nabla g(x)\| = \|P \nabla f(x)\| \leq \mathrm{Lip}(f)$)
\[
| \nabla f(z_i) \cdot x_i' | = | \nabla g(Pz_i) \cdot (Px_i') | \leq \mathrm{Lip}(f) \cdot \|P x_i'\| \,.
\]
Combining the two above displays one has:
\[
\frac{n}{\mathrm{Lip}(f)} \leq \sum_{i=1}^n \|P x_i'\| \leq \sqrt{n \sum_{i=1}^n \|P x_i'\|^2} = \sqrt{n \sum_{i=1}^n x_i'^{\top} P^{\top} P x_i'} = \sqrt{n \langle \sum_{i=1}^n x_i' x_i'^{\top}, P^{\top} P \rangle_{\mathrm{HS}}} \,.
\]
Using [Theorem 5.39, \cite{vershynin12}] (specifically (5.23)) we know that with probability at least $1-\exp(C - c d)$ we have $\left\| \sum_{i=1}^n x_i' x_i'^{\top} \right\|_{\mathrm{op}} \leq C \frac{n}{d}$ (here we use $n \geq d$ too). Moreover we have $\|P^{\top} P\|_{\mathrm{op},*} = \mathrm{Tr}(P^{\top} P) = \mathrm{Tr}(P P^{\top}) = k$. Thus we have $\langle \sum_{i=1}^n x_i' x_i'^{\top}, P^{\top} P \rangle_{\mathrm{HS}} \leq C \frac{n \cdot k}{d}$ so that with the above display one obtains:
\[
\frac{n}{\mathrm{Lip}(f)} \leq n \sqrt{\frac{C k}{d}} \,,
\]
which concludes the proof.
\end{proof}

\subsection{Power activation} \label{sec:tensor}
We prove here the conjecture for the power activation function $\psi(t) = t^p$ with $p$ an integer and with no bias terms (we deal with general polynomials, including with bias, in the next section). Without bias such a network can be written as:
\begin{equation} \label{eq:tensornetwork}
f(x) = \sum_{\ell=1}^k a_{\ell} (w_{\ell} \cdot x)^p = \langle T, x^{\otimes p} \rangle \,,
\end{equation}
where $T = \sum_{\ell=1}^k a_{\ell} w_{\ell}^{\otimes p}$. As we already saw in the proof of Theorem \ref{thm:tensorUB} (see specifically [Appendix \ref{app1}, Lemma \ref{lem:tensordecomposition}]), without loss of generality we have $k \leq C_p d^{p-1}$. We now prove that tensor networks of the form \eqref{eq:tensornetwork} cannot obtain a Lipschitz constant\footnote{Note that without loss of generality one can assume $T$ to be symmetric, since we only consider how it acts on $x^{\otimes p}$. For symmetric tensors one has that the Lipschitz constant on the unit ball is lower bounded by the operator norm of $T$ thanks to \eqref{eq:symnorm}} better than $\sqrt{n/d^{p-1}}$, in accordance with Conjecture \ref{conj:LB} for full rank tensors (where $k \approx d^{p-1}$).

\begin{theorem} \label{thm:tensorLB}
Assume that we have a tensor $T$ of order $p$ such that
\[
\langle T, x_i^{\otimes p} \rangle = y_i, \forall i \in [n] \,.
\]
Then, for generic data, with probability at least $1- C \exp(- c_p d)$, one must have
\[
\|T\|_{\mathrm{op}} \geq c_p  \sqrt{\frac{n}{d^{p-1}}} \,.
\]
\end{theorem}

\begin{proof}
Denoting $\Omega = \sum_{i=1}^n y_i x_i^{\otimes p}$, we have (using $y_i^2 =1$ for the first equality and [Appendix \ref{app1}, Lemma \ref{lem:tensornorm}] for the last inequality):
\begin{equation} \label{eq:tensorLB}
n = \langle T, \Omega \rangle \leq \|\Omega\|_{\mathrm{op}} \cdot \|T\|_{\mathrm{op},*} \leq d^{p-1} \cdot \|\Omega\|_{\mathrm{op}} \cdot \|T\|_{\mathrm{op}} \,.
\end{equation}
Thus we obtain $\|T\|_{\mathrm{op}} \geq \frac{n}{d^{p-1} \cdot \|\Omega\|_{\mathrm{op}}}$, and it only remains to apply [Appendix \ref{app2}, Lemma \ref{lem:operatornormtensor}] which states that with probability at least $1- C \exp(- c_p d)$ one has $\|\Omega\|_{\mathrm{op}} \leq C_p \sqrt{\frac{n}{d^{p-1}}} $. 
\end{proof}

\subsection{Polynomial activation} \label{sec:polynomial}
We now observe that one can generalize Theorem \ref{thm:tensorLB} to handle biases, and in fact even general polynomial activation function. Indeed, observe that any polynomial of $\langle w,x\rangle - b$ must also be a polynomial in $\langle w,x \rangle$, albeit with different coefficients. 

\begin{theorem} \label{thm:polynomialLB}
Let $\psi(t) = \sum_{q=0}^p \alpha_q t^q$ and assume that we have $f \in \cF_k(\psi)$ such that $f(x_i) = y_i, \forall i \in [n]$. Then, for generic data, with probability at least $1-C \exp(- c_p d)$ one must have
\[
\mathrm{Lip}_{\{x : \|x\| \leq 1\}} (f) \geq c_p \sqrt{\frac{n}{d^{p-1}}} \,.
\]
\end{theorem}

\begin{proof}
Note that for $f \in \cF_k(\psi)$ there exists tensors $T_0, \hdots, T_p$, such that $T_q$ is a tensor of order $q$, and $f$ can be written as:
\[
f(x) = \sum_{q=0}^p \langle T_q, x^{\otimes q} \rangle \,.
\]
Now let us define $\Omega_q = \sum_{i=1}^n y_i x_i^{\otimes q}$, and observe that
\[
n = \sum_{i=1}^n y_i f(x_i) = \sum_{q=0}^p \langle T_q, \Omega_q \rangle \,,
\]
and thus there exists $q \in \{1,\hdots, p\}$ such that $\langle T_q, \Omega_q \rangle \geq c_p n$ (we ignore the term $q=0$ by considering the largest balanced subset of the data, i.e. we assume $\sum_{i=1}^n y_i = 0$). Now one can repeat the proof of Theorem \ref{thm:tensorLB} to obtain that  with probability at least $1-C \exp(- c_p d)$, one has $\|T_q\|_{\mathrm{op}} \geq c_p  \sqrt{\frac{n}{d^{p-1}}}$. It only remains to observe that the Lipschitz constant of $f$ on the unit ball is lower bounded by $ \|T_q\|_{\mathrm{op}}$.

As we mentioned in Section \ref{sec:tensor}, without loss of generality we can assume $T_q$ is symmetric, and thus by \eqref{eq:symnorm} there exists $x \in \mS^{d-1}$ such that $\|T_q\|_{\mathrm{op}} = \langle T_q, x^{\otimes q} \rangle$. Now consider the univariate polynomial $P(t) = f(t x)$. By Markov brothers' inequality one has $\max_{t \in [-1,1]} P(t) \geq |P^{(q)}(0)| = q! \cdot |\langle T_q, x^{\otimes q} \rangle| = q! \cdot \|T_q\|_{\mathrm{op}}$, thus concluding the proof.
\end{proof}

\subsection{Quadratic activation} \label{sec:quadratic}
In Section \ref{sec:tensor} we obtained a lower bound for tensor networks that match Conjecture \ref{conj:LB} only when the rank of the corresponding tensor is maximal. Here we show that for quadratic networks (i.e., $p=2$) we can match Conjecture \ref{conj:LB}, and in fact even obtain a better bound, for any rank $k$:

\begin{theorem} \label{thm:quadratic}
Assume that we have a matrix $T \in \R^{d \times d}$ with rank $k$ such that:
\[
\langle T, x_i^{\otimes 2} \rangle = y_i, \forall i \in [n] \,.
\]
Then, for generic data, with probability at least $1- C \exp(-  c d)$, one must have
\[
\|T\|_{\mathrm{op}} \geq c  \frac{\sqrt{n d}}{k} \ \ ( \geq c \sqrt{n / k} ) \,.
\]
\end{theorem}

\begin{proof}
The proof is exactly the same as for Theorem \ref{thm:tensorLB}, except that in \eqref{eq:tensorLB}, instead of using Lemma \ref{lem:tensornorm} we use the fact that for a matrix $T$ of rank $k$ one has:
\[
\|T\|_{\mathrm{op},*} \leq k \cdot \|T\|_{\mathrm{op}} \,.
\]
\end{proof}

\section{Experiments} \label{sec:exp}
We consider a generic dataset from the Gaussian model (i.e., $x_1, \hdots, x_n$ i.i.d. from $\mathcal{N}(0, \tfrac{1}{d}I_d)$ and labels $y_1, \hdots, y_n$  i.i.d from the uniform distribution over $\{-1,1\}$ and independent of $x_1,\dots, x_n$). For various values of $(n,d,k)$ we train two-layers neural networks with $k$ $\mathrm{ReLU}$ units and batch normalization (see \cite{ioffe2015batch}) between the linear layer and $\mathrm{ReLU}$ layer, using the Adam optimizer \citep{kingma2014adam} on the least squares loss. We keep the values of $(n,k,d)$ where the network successfully memorizes the random labels (possibly after a rounding to $\{-1,+1\}$, and such that prior to rounding the least squares loss is at most some small value $\epsilon$ to be specified later). Given a triple $(n,d,k)$, suppose the output of the trained network is $f_{n,d,k} : \mathbb{R}^d \to \mathbb{R}$. We then generate $z_1,\dots, z_T$ (where $T=1000$) i.i.d from the distribution $\mathcal{N}(0,\tfrac{1}{d}I_d)$, independently of everything else and define the ``maximum random gradient'' to be $\max_{i\in [T]}\|\nabla f_{n,k,d}(z_i)\|$ (it is our proxy for the true Lipschitz constant $\sup_{z \in \mS^{d-1}} \|\nabla f_{n,d,k}(z)\|$). Our experimental results are as follows:

\paragraph{Experiment 1:} We ran experiments with $n$ between $100$ and $2000$, $d$ between $\sim 50$ and $\sim n$, and $k$ between $\sim 10$ and $\sim n$ (we also choose $\epsilon=0.02$ for the thresholding). In Figure~\ref{fig:scatter_plot_1} we give a scatter plot of $\left(\sqrt{\frac{n}{k}}, \max_{i\in [T]}\|\nabla f_{n,k,d}(z_i)\| \right)$, and as predicted we see a linear trend, thus providing empirical evidence for Conjecture~\ref{conj:LB}.

\begin{figure}
    \centering
    \begin{minipage}{0.45\textwidth}
        \centering
        \includegraphics[width=0.9\textwidth]{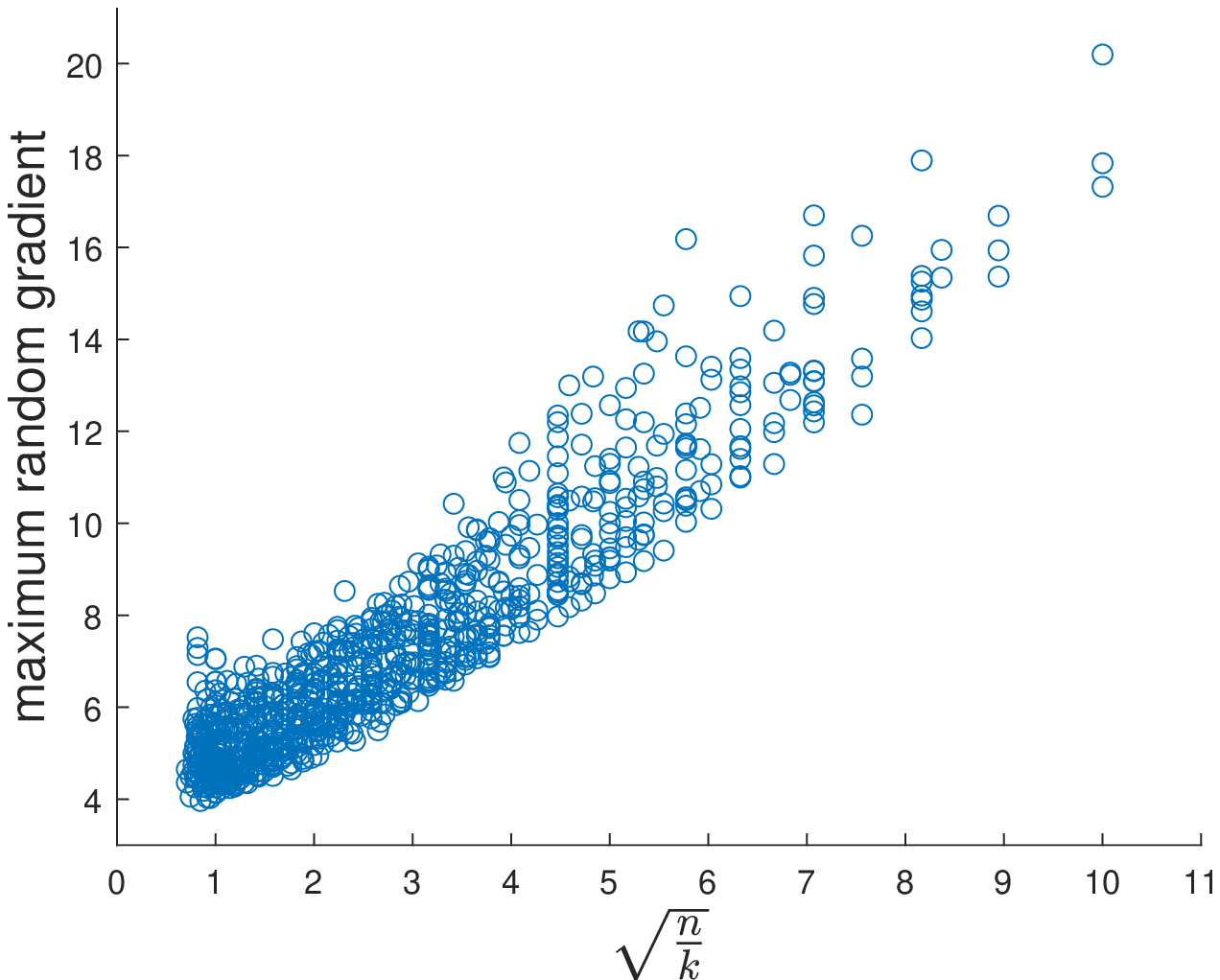} 
        \caption{Scatter plot of maximum random gradient with respect to $\sqrt{\frac{n}{k}}$ with $906$ data points (Experiment 1)}
        \label{fig:scatter_plot_1}
    \end{minipage}\hfill
    \begin{minipage}{0.45\textwidth}
        \centering
        \includegraphics[width=0.9\textwidth]{scatter_plot_2} 
        \caption{Scatter plot of maximum random gradient with respect to $\sqrt{d}$ in optimal smoothness (blue) and optimal size (red) regimes (Experiment 2)}
        \label{fig:scatter_plot_2}
    \end{minipage}
\end{figure}

\paragraph{Experiment 2:}

In this experiment, we investigate the two extreme cases $k \sim n$ and $k \sim n/d$. We fix $n = 10^4$ and sweep the value of $d$ between $10$ to $5000$ (we also choose $\epsilon=0.1$ for the thresholding). In the first case, we let $k = n$ and in the second case we let $k = 10n/d$. In Figure~\ref{fig:scatter_plot_2} we plot $\sqrt{d}$ versus the maximum random gradient (as defined above) for both cases. We observe a linear dependence between the maximum gradient value and $\sqrt{d}$ when we have $k = 10n/d$, and roughly a constant maximum gradient value when $k=n$, thus providing again evidence for Conjecture~\ref{conj:LB}

\bibliographystyle{plainnat}
\bibliography{neuralbib}

\appendix
\section{Results on tensors} \label{app1}
A tensor of order $p$ is an array $T = (T_{i_1, \hdots, i_p})_{i_1, \hdots, i_p \in [d]}$. The Frobenius inner product for tensors is defined by:
\[
\langle T, S \rangle = \sum_{i_1,\hdots, i_p=1}^d T_{i_1,\hdots,i_p} S_{i_1, \hdots, i_p} \,,
\]
with the corresponding norm $\|\cdot\|_{\mathrm{F}}$. A tensor is said to be of rank $1$ if it can be written as:
\[
T= u_1 \otimes \hdots \otimes u_p \,,
\]
for some $u_1, \hdots, u_p \in \R^d$. The operator norm $\|\cdot\|_{\mathrm{op}}$ is defined by:
\[
\|T\|_{\mathrm{op}} = \sup_{S \text{ rank } 1, \|S\|_{\mathrm{F}} \leq 1} \langle T, S\rangle \,.
\]
For symmetric tensors (i.e., such that the entries of the array are invariant under permutation of the $p$ indices), Banach's Theorem (see e.g., [(2.32), \cite{nemirovski2004interior}]) states that in fact one has
\begin{equation} \label{eq:symnorm}
\|T\|_{\mathrm{op}} = \sup_{x \in \mS^{d-1}} \langle T, x^{\otimes p} \rangle \,.
\end{equation}
We refer to \cite{friedland2018nuclear} for more details and background on tensors. We now list a couple of useful results, with short proofs.

\begin{lemma} \label{lem:liptensor}
For a tensor $T$ of order $p$, one has 
\[
\mathrm{Lip}_{\mS^{d-1}}(x \mapsto \langle T, x^{\otimes p} \rangle) \leq p \cdot \|T\|_{\mathrm{op}} \,. 
\]
\end{lemma}

\begin{proof}
One has for any $x, y \in \mS^{d-1}$,
\begin{eqnarray*}
\left|\langle T, x^{\otimes p} \rangle - \langle T, y^{\otimes p} \rangle \right|  & \leq & \sum_{q=1}^p \left|\langle T, x^{\otimes p-q+1} \otimes y^{\otimes q-1} \rangle - \langle T, x^{\otimes p-q} \otimes y^{\otimes q} \rangle \right| \\
& \leq & p \cdot \|x-y\| \cdot \sup_{x^1, \hdots, x^p \in \mS^{d-1}} \left|\langle T, \otimes_{q=1}^p x^q \rangle \right| \\
& = & p \cdot \|x-y\| \cdot \|T\|_{\mathrm{op}} \,.
\end{eqnarray*}
\end{proof}

\begin{lemma} \label{lem:tensordecomposition}
For any tensor $T$ of order $p$, there exists $w_1, \hdots, w_{2^p d^{p-1}} \in \R^d$ and $\xi_1, \hdots, \xi_{2^p d^{p-1}} \in \{-1,+1\}$ such that for all $x \in \R^d$,
\[
\langle T, x^{\otimes p} \rangle = \sum_{\ell=1}^{2^p d^{p-1}} \xi_{\ell} \cdot (w_{\ell} \cdot x)^p \,.
\]
\end{lemma}

Results like Lemma \ref{lem:tensordecomposition} go back at least to \cite{reznick1992sum}. In fact much more precise results on minimal decomposition in rank-$1$ tensors are known thanks to the work of \cite{alexander1995polynomial}. We refer to \citep{comon2008symmetric} for more discussion on this topic.

\begin{proof}
First note that trivially $T$ can be written as:
\begin{equation} \label{eq:orthdecompotensor}
T = \sum_{i_1, \hdots, i_{p-1} =1}^d e_{i_1} \otimes \hdots \otimes e_{i_{p-1}} \otimes T[i_1, \hdots, i_{p-1}, 1:d] \,.
\end{equation}
Thus one only needs to prove that a function of the form
$x \mapsto \prod_{q=1}^p (w_q \cdot x)$ can be written as the sum of $2^p$ functions of the form $(w' \cdot x)^p$. To do so note that, with $\epsilon_q$ i.i.d. random signs,
\[
\E \left[ \prod_{q=1}^p \epsilon_q \cdot \left(\sum_{q=1}^p \epsilon_q w_q \cdot x \right)^p \right] = \E \left[ \prod_{q=1}^p \epsilon_q \cdot \sum_{q_1,\hdots, q_p =1}^p \left( \prod_{r=1}^p \epsilon_{q_r} w_{q_r} \cdot x \right)  \right]  = p! \prod_{q=1}^p (w_q \cdot x) \,.
\]
\end{proof}

\begin{lemma} \label{lem:tensornorm}
For any tensor $T$ of order $p$ one has:
\[
\|T\|_{\mathrm{op},*} \leq d^{p-1} \cdot \|T\|_{\mathrm{op}} \,.
\]
\end{lemma}

The above result and its proof are directly taken from \cite{li2018orthogonal}. We only repeat the argument here for sake of completeness.

\begin{proof}
Note that the decomposition \eqref{eq:orthdecompotensor} is orthogonal, and thus for any tensor $S$ of order $p$ one has:
\begin{eqnarray*}
\langle T, S \rangle & \leq & \sqrt{d^{p-1} \cdot \sum_{i_1, \hdots, i_{p-1} =1}^{d} \langle e_{i_1} \otimes \hdots \otimes e_{i_{p-1}} \otimes T[i_1, \hdots, i_{p-1}, 1:d], S \rangle^2 } \\
& \leq & \sqrt{d^{p-1} \cdot \|S\|_{\mathrm{op}}^2 \cdot \sum_{i_1, \hdots, i_{p-1} =1}^{d} \|T[i_1, \hdots, i_{p-1}, 1:d]\|^2 } \\
& = & d^{\frac{p-1}{2}} \cdot \|S\|_{\mathrm{op}} \cdot \|T\|_F \,.
\end{eqnarray*}
Thus one has $\|T\|_{\mathrm{op},*} \leq d^{\frac{p-1}{2}} \cdot \|T\|_F$. By duality one also has $\|T\|_{\mathrm{op}} \geq d^{- \frac{p-1}{2}} \cdot \|T\|_F$, which concludes the proof.

\end{proof}

\section{Results on random tensors} \label{app2}

\begin{lemma} \label{lem:conctensor}
For any fixed $x \in \mS^{d-1}$ and generic data, with probability at least $1- C \exp(- c_p \tau)$ one has:
\[
\left| \sum_{i=1}^n y_i (x_i \cdot x)^p \right| \leq C_p \sqrt{\frac{n \tau}{d^{p}}} \,.
\]
\end{lemma}

\begin{proof}
Using [Theorem 1, \cite{paouris2017random}] one has, for any fixed $x \in \mS^{d-1}$ and $\tau \leq n$,
\[
\P\left( \left| d^{p/2} \sum_{i=1}^n |x_i \cdot x|^{p} - n \sigma_p \right| > C_p \sqrt{n \tau} \right) \leq C \exp(- c_p \tau) \,,
\]
where $\sigma_p$ denotes the $p^{th}$ moment of the standard Gaussian.
Let us denote $n^+ = |\{i \in [n] : y_i = +1\}$ and $T^+ = \sum_{i : y_i = +1} x_i^{\otimes p}$, and similarly for $n^-, T^-$. Now with probability $1- C \exp(- c \tau)$ (with respect to the randomness of the $y_i's$) we have 
\[
|n^+ - n^-| \leq \sqrt{n \tau} \,.
\]
Thus combining the two above displays we obtain with probability at least $1- C \exp(- c_p \tau)$,
\[
d^{p/2} \left| \sum_{i : y_i = +1} |x_i \cdot x|^{p}  - \sum_{i : y_i = -1} |x_i \cdot x|^{p} \right| \leq C_p \sqrt{n \tau} + \sigma_p |n^+ - n^-| \leq C_p \sqrt{n \tau} \,,
\]
\end{proof}

\begin{lemma} \label{lem:operatornormtensor}
For generic data, with probability at least $1- C \exp(- c_p d)$ one has:
\[
\left\| \sum_{i=1}^n y_i x_i^{\otimes p} \right\|_{\mathrm{op}} \leq C_p \sqrt{\frac{n}{d^{p-1}}} \,.
\]
\end{lemma}

\begin{proof}
Let $\cN$ be an $\frac{1}{2p}$-net of $\cS^{d-1}$ (in particular $|\cN| \leq C_p^d$). By an union bound and Lemma \ref{lem:conctensor} one has:
\begin{equation} \label{eq:highprob}
\P\left(\exists x \in \cN_{\epsilon} \ : \ \left|\sum_{i=1}^n y_i |x_i \cdot x|^{p} \right| > C_p \sqrt{\frac{n}{d^{p-1}}} \right) \leq C \exp(- c_p d) \,,
\end{equation}

Let $T = \sum_{i=1}^n y_i x_i^{\otimes p}$. Note that $T$ is symmetric, and thus thanks to \eqref{eq:symnorm} and Lemma \ref{lem:liptensor}, one has:
\[
\|T\|_{\mathrm{op}} \leq \max_{x \in \cN} \langle T, x^{\otimes p} \rangle + \frac{1}{2} \|T\|_{\mathrm{op}} \,,
\]
and in particular $\|T\|_{\mathrm{op}} \leq 2 \max_{x \in \cN} \langle T, x^{\otimes p} \rangle $, which together with \eqref{eq:highprob} concludes the proof.
\end{proof}

\end{document}